\documentclass[10pt, conference]{IEEEtran}

\usepackage{amssymb}
\usepackage{amsfonts}
\usepackage{amsmath,amsthm}
\usepackage{bbm}
\usepackage[dvips]{graphicx}
\usepackage{cite}
\usepackage{balance}
\usepackage{color}
\usepackage{xcolor}
\usepackage{epstopdf}
\usepackage{algorithm}
\usepackage{algorithmicx}
\usepackage{algpseudocode}
\usepackage[switch]{lineno}
\usepackage{bm}
\usepackage{wrapfig}
\usepackage{url}
\usepackage{multirow}
\usepackage{mathtools}
\usepackage{tcolorbox}
\usepackage{enumitem}
\usepackage{array}
\usepackage{amsmath}

\usepackage{mdframed}

\def\BibTeX{{\rm B\kern-.05em{\sc i\kern-.025em b}\kern-.08em
		T\kern-.1667em\lower.7ex\hbox{E}\kern-.125emX}}

\newtheorem{theorem}{Theorem}

\graphicspath{{Figures/}}

\def\c{\mathcal}
\def\t{\theta}
\def\g{\nabla}
\def\sumT{\sum_{t=1}^T}
\def\summ{\sum_{i=1}^m}

\newcommand{\gi}[1]{\textcolor{orange}{[g: -- #1]}}

\IEEEoverridecommandlockouts
\begin{document}
\title{Fair and Online Multi-Task Learning}
\title{Fair Online-within-Online Multi-Task Learning for AI-RANs}
\title{Fair and Adaptive Multi-Task Learning for AI-RANs}
\title{Equitable Multi-Task Learning for AI-RAN Systems}
\title{Equitable Multi-Task Learning for AI-RANs}
\author{
	\IEEEauthorblockN{Panayiotis Raptis, Fatih Aslan, George Iosifidis}
	\IEEEauthorblockA{Department of Software Technology, Delft University of Technology, The Netherlands}
    \IEEEauthorblockA{Email: \{p.raptis, f.aslan, g.iosifidis\}@tudelft.nl}
}
\maketitle

\begin{abstract}
AI–enabled Radio Access Networks (AI-RANs) are expected to serve heterogeneous users with time-varying learning tasks over shared edge resources. Ensuring equitable inference performance across these users requires adaptive and fair learning mechanisms. This paper introduces an Online-Within-Online Fair Multi-Task Learning (OWO-FMTL) framework that ensures long-term equity across users. The method combines two learning loops: an outer loop updating the shared model across rounds and an inner loop re-balancing user priorities within each round with a lightweight primal–dual update. Equity is quantified via generalized $\alpha$-fairness, allowing trading-off efficiency and fairness. The framework guarantees diminishing performance disparity over time and operates with low computational overhead suitable for edge deployment. Experiments on convex and deep-learning tasks confirm that OWO-FMTL outperforms existing multi-task learning baselines under dynamic scenarios.
\end{abstract}

\vspace{1mm}
\begin{IEEEkeywords}
	AI-RAN, Multi-Task Learning, Fairness.
\end{IEEEkeywords}

\section{Introduction} \label{sec:intro}

\textbf{Motivation}. One of the key goals of NextGen mobile networks is to support AI services for their users. These services provide high-accuracy low-latency inferences by processing data flows with machine learning (ML) models. Prominent examples include AR/VR, cognitive assistance and autonomous driving \cite{10415393}, \cite{10636754}. An effective way to support these services is with edge computing, where the ML models are deployed at the Radio Access Network (RAN), so as to reduce latency (one hop connections) and save bandwidth. The AI-RAN alliance \cite{ai-ran-alliance} promotes the native integration of AI and RANs that points exactly in the direction of such solutions \cite{nokia-ai-ran}. Alas, as edge resources are limited, this appealing idea remains largely impractical in real-world deployments. 

A potential solution to this problem is Multi-Task Learning (MTL), a powerful paradigm where a model is jointly trained by, and for, different tasks~\cite{mtl-survey}. This consolidation economizes resources and can improve the ML performance through latent learning among the tasks. In the context of AI-RAN, MTL can be used to support many users with a single model, as opposed to training and deploying a different ML model per user. For instance, users can run their video analytic (VA) tasks on a common deep learning model, despite their different contexts (e.g., frame background) or goals (e.g. identify humans or cars). Similarly, a single model can be used for different applications altogether (e.g., RAN management functions).

However, the adoption of MTL requires solving a key challenge: ensure that the users receive fair inference performance or, put differently, \emph{AI equity}. When using training data from different tasks, it is likely one of them to dominate the optimization process and thus lead to an imbalanced model that performs very well for its purpose at the expense of others. Technically, this happens due to conflicting gradients, that renders default average loss-based training unfit~\cite{yu2020gradient}. This issue is further compounded in AI-RAN where the users' tasks and surrounding conditions can change rapidly (e.g., the VA context evolves), and hence the model must adapt through a intertwined training/inference process while remaining fair. 

\emph{This work aims to address this issue with a principled and practical MTL mechanism that ensures long-term fairness in inference performance (AI equity) of AI-RAN systems.}

\textbf{Related Work}. Federated Learning (FL) solutions have been extensively explored for future networked systems in recent years. In contrast, MTL remains relatively understudied, even though it also addresses learning from multiple related data sources. However, the two paradigms differ fundamentally in their objectives and communication settings. FL focuses on collaboratively training a single global model across decentralized clients \cite{mcmahan2017communication}, while MTL aims to improve generalization by jointly learning multiple related tasks through shared representations \cite{mtl-survey}. Although hybrid approaches integrate MTL principles into federated frameworks for personalization \cite{federated-mtl-nips17,mohri2019agnostic}, they remain fundamentally FL-based.

Still, MTL on its own has already shown significant promise across various applications. For instance, it was found to improve RAN-related tasks such as carrier prediction, resource management, modulation and signal classification, as well as radio environment and mobility characterization \cite{9144137,9627159, JAGANNATH2022101793, 9382101, ericsson-ai-native-mtl}. Further, MTL has been used for multi-objective network optimization, collaborative control and multi-task edge inference, demonstrating that shared model representations can enhance learning \cite{shisher2025computation, wang2025multi}. A natural next step is to use MTL for dynamic user tasks. Online learning (OL) is well suited to this setting and fairness-aware Online Convex Optimization (OCO) has been already explored for O-RAN workload assignment and energy minimization \cite{Fatih2024,Fatih2025}. However, its integration with MTL for dynamic user-task learning remains unstudied.

Balanced training in MTL has been previously studied in ML \cite{yu2020gradient}, but the issue of fairness was considered only recently \cite{fair-icml24}. These works study the offline version of the problem, where the task goals and data are static, and the model is trained one-off. The AI-RAN MTL problem however calls for continuous model training using feedback from each task's performance, so as to adapt to users' evolving needs. In technical terms, this requires dynamic algorithms that guarantee fairness over the life-cycle of tasks, instead of momentarily, cf. \cite{bib:tareq_fairness}. For this kind of perpetual model adaptation, Online-Within-Online (OWO) \cite{point-neurips19} and Meta-Learning (MeL) approaches \cite{pmlr-v97-balcan19a}, \cite{pmlr-v97-finn19a} have been proven successful. Yet, none of these works considers multiple users or fairness. 

\textbf{Contributions}. Motivated by the above, we study an AI-RAN system (Fig.~\ref{fig:algorithm-mechanics}) where users create tasks dynamically over time rounds, and each task consists of many jobs (one per slot). We consider a U-type split-learning architecture where the users' data and labels are kept locally, and the RAN server hosts the intermediate shared model \cite{tirana2024workflow}. The goal is ensure the inference performance is fair for all users during each round, i.e., to achieve round-long fairness for all jobs. This is a joint inference and training setting where the users train the model while using it. We make no assumptions about the training data distributions or the scope of tasks, both of which are initially unknown and change arbitrarily over time.

The operation of this system is modeled as an OWO fairness problem that we tackle with a two-layer learning algorithm based on the theory of OCO. At the higher layer (outer-loop), the system learns to initialize the model at the start of each round; where a good initialization helps the model to adapt fast to the forthcoming jobs. At the lower layer (inner-loop) the system updates the model with the feedback\footnote{Feedback may come from a small test set \cite{pmlr-v97-balcan19a, pmlr-v97-finn19a}, direct measurements \cite{ericsson-ai-native-mtl}, or self/unsupervised metrics.} the users provide after each slot's jobs are executed. The algorithm offers optimality guarantees or, formally, zero fairness regret: it achieves  the same performance as a model that would have been selected by a clairvoyant oracle, and this holds for any possible sequence of jobs and tasks. At the same time, the algorithm has minimal computing and storage requirements, which is in stark contrast to other MTL methods that calculate and store as many gradients as the tasks. Our extensive performance evaluation using both (convex) kernel regression and (non-convex) deep learning experiments, verifies the practical gains of the algorithm over several benchmarks and state-of-the-art MTL solutions. 

In summary, the contributions of this work are: \emph{(i)} Introduces, for the first time, the problem of dynamic multi-task fairness in AI-RAN / Edge computing systems for multiple users and dynamically-arriving tasks; \emph{(ii)} Models this process as an OWO (two-layer learning) problem and designs a primal-dual scalable algorithm that guarantees zero fairness regret for a range of perturbation scenarios, including adversarial ones; \emph{(iii)} Evaluates the algorithm using kernel regression and deep-learning experiments, compares it with state-of-the-art MTL solutions, and includes an ablation study.

\section{System Model and Problem Statement}\label{sec:system-model}

\textbf{Model}. We consider an AI-RAN system serving a set of users $\c K$, who create ML tasks in a streaming fashion. The system operates in \emph{mixed-time} scale with $m$ slots per round $t\!\leq\!T$. During each round $t$, every user $k\in \c K$ needs to perform a task, denoted $(k,t)$, that involves $m$ sequential jobs (one per slot). The system supports the users through a shared model which is trained dynamically with data from all tasks.
For instance, consider a mobile VM service running at the edge~\cite{automl}. Each user $k$ corresponds to a mobile device with the task of object detection through a video stream.
A single task $(k, t)$ then represents one usage session of the service under a particular context, e.g., lighting and weather conditions.

\begin{figure}[t]
\centering
\begin{minipage}{0.45\columnwidth}
  \centering
  \includegraphics[width=\linewidth]{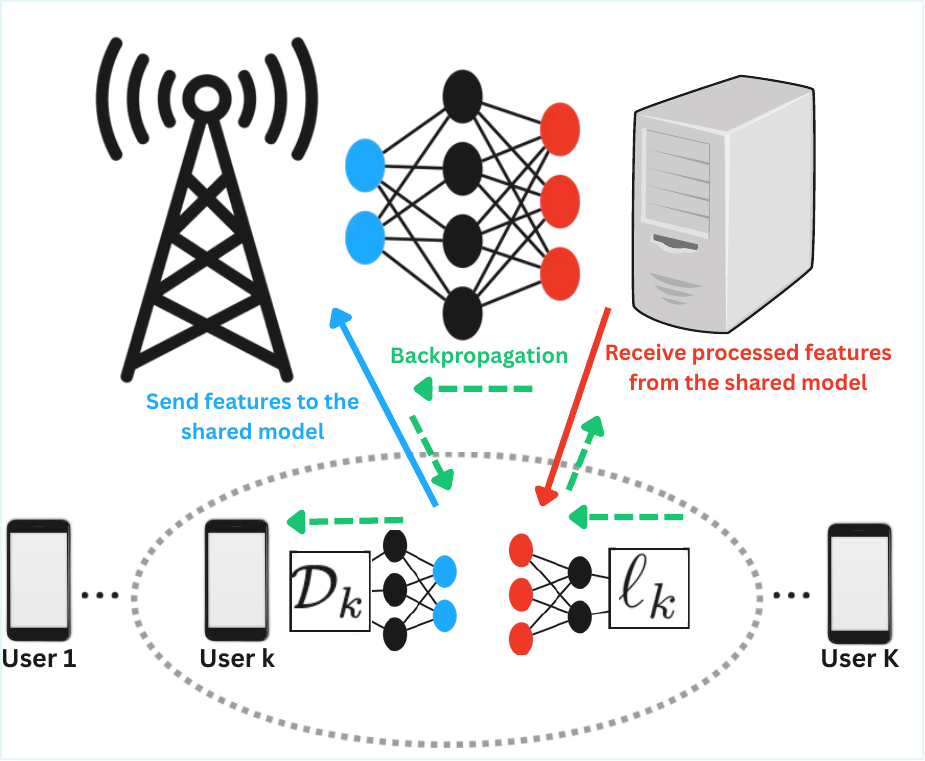}
\end{minipage}
\begin{minipage}{0.45\columnwidth}
  \centering
  \includegraphics[width=\linewidth]{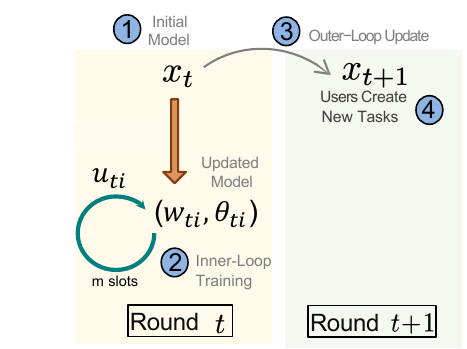}
\end{minipage}
\caption{Illustration of system's operation, where at each slot of a round all the participating users transmit the locally extracted features from their own data to the shared model. The shared model processes these features and send the results to the corresponding user, enabling them to complete their forward pass locally. Afterward, each user performs the backward pass, and the shared model aggregates and appropriately weights the received gradients to update its parameters (inner-loop). Once all slots within a round are completed, the shared model performs an additional autonomous update, preparing it for the next round (outer-loop).}
\vspace{-5mm}
\label{fig:algorithm-mechanics}
\vspace{-0.8em}
\end{figure}

Within a task, the $m$ sequential slots correspond to successive video frames processed over time. As contexts evolve, the characteristics of successive tasks can change substantially, and users themselves exhibit diverse usage patterns. The AI-RAN system must therefore support all users through a single shared model by balancing their training gains.
In order to capture this multi-round system operation, we employ the lifelong extension of MTL, known also as OWO learning \cite{point-neurips19}, and consider two levels of learning:
\vspace{-0.5em}
\begin{itemize}
\item[\emph{(i)}] an \emph{outer-loop} (meta-update), which learns how to initialize the shared model at the start of each round, and
\item[\emph{(ii)}] an \emph{inner-loop}, which updates the shared model after each slot, based on users' feedback.
\vspace{-0.5em}
\end{itemize}
The initialization helps the system pick a model that adapts fast to the jobs of each round with only few inner iterations.

Let $x_t \!\in\!\mathcal{X} \! \subset \! \mathbb{R}^d$ denote \emph{initialization} of the shared model of $d$ parameters at round $t$, and let $\theta_{ti} \in \Theta \subset \mathbb{R}^d$ denote the model itself at the start of slot $i$ within round $t$, i.e., slot $(t,i)$. The domains $\mathcal{X}$ and $\Theta$ may coincide, or represent different regions of the parameter space; this is a design choice. Once the users’ jobs at slot $(t,i)$ are completed, the model is updated to $\theta_{t,i+1}$ via the {inner-loop} learning. Both the {outer} and the {inner-loop} learning schemes are performance-driven. In particular, each job $(t,i,k)$ is associated with a concave \emph{utility} function $u_{tik}: \Theta \mapsto \mathbb{R}_+$, which captures the user-perceived \emph{training gain} and can be generalized across diverse AI-RAN scenarios, provided that its concavity is preserved. The corresponding utility vectors are then denoted by $u_{ti} = (u_{tik},\, k \in \mathcal{K})$. Among possible choices, we select $u_{tik}(\theta_{ti}) = \ell^{\max}_k - \ell_{tik}(\theta_{ti})$, which are defined with respect to a high-loss model that user $k$ would have incurred if they had not been supported by the system. In practice, if this value is unknown, one can simply use a sufficiently large offset. Under the standard assumption of convex losses $\{\ell_{tik}\}_{t,i,k}$ \cite{khodak-neurips19, pmlr-v97-finn19a}, the resulting utilities $\{u_{ti}\}_{t,i}$ are still concave.

 
\textbf{Problem Statement}. Our aim is to design a learning mechanism that trains the shared model dynamically (i.e., while it is being employed) in a way that ensures the users enjoy fair performance. To that end, we adopt the $\alpha$-fairness metric \cite{ATKINSON1970244}, where the choice of $\alpha$ determines the notion of fairness, and, for a vector $u=(u_k, k\in\mathcal K)$, it is:
\begin{align}\label{eq:basic-fairness}
F_{\alpha}(u) \doteq 
\begin{cases}
  \sum_{k \in \mathcal{K}} \frac{u_k^{1-\alpha} - 1}{1-\alpha}, & \alpha \in \mathbb{R}_{\geq 0} \setminus \{1\}, \\[1.2ex]
  \sum_{k \in \mathcal{K}} \log(u_k), & \alpha = 1 .\quad
\end{cases}
\end{align}

We argue that the users should receive \emph{equitable} performance across the \emph{learning horizon} of their $m$ jobs within each round, and therefore we apply this metric to their round-aggregated utilities. This goal is technically more compounded but performance-wise superior to myopically enforcing fairness for each slot separately. Indeed, by the concavity of $F_\alpha$, Jensen’s inequality guarantees the fairness of utilities averaged over the round surpasses the average fairness of slot utilities. Besides, enforcing fairness over the average utilities creates room for efficiency and thus leads to lower PoF, cf. \cite{bib:tareq_fairness}. 
Conversely, optimizing the fairness over the entire horizon $T$ causes unfairness amongst rounds, although it has lower PoF.


Nevertheless, guaranteeing this \emph{per-round} fairness is highly non-trivial. Firstly, observe that when the model is updated in each slot, the system (henceforth, the \emph{learner}) is not aware of the users' utilities, as they depend on the still-unknown users' evaluation data and jobs. Even worse, this information is unknown for the entire batch of $m$ slots at the start of each round, where the learner decides the initialization $x_t$. Observe also that, since we allow the utility functions (or the data distributions) to change arbitrarily, this lack of information cannot be compensated through stochasticity as in {offline ML}. Hence, we need an \emph{online} model adaptation tool, namely OCO~\cite{orabona-book}, that is oblivious to this information and robust to non-stationary inputs. Secondly, achieving fairness requires  balancing the users utilities throughout the round ($m$ slots), rather than optimizing at each slot in isolation. This limitation renders off-the-shelf OCO algorithms inadequate, as they typically optimize independently each decision.

Motivated by the above, we treat \emph{model initialization} and \emph{inner-loop} updates as intertwined learning problems to achieve sublinear \emph{Round-Average Fairness} (RAF) regret:
\begin{align}
\c R_T\!=\!\frac{1}{T}\sum_{t=1}^T \left[F_\alpha\left( \frac{1}{m}\sum_{i=1}^{m} u_{ti}(\t_{t}^\star)\right)  \! - \! F_\alpha\left( \frac{1}{m}\sum_{i=1}^{m} u_{ti}(\t_{ti})  \right)\right], \notag
\end{align}
where $\t_t^\star$ denotes the \emph{fairest} model for round $t$, or equivalently:
\vspace{-1em}
\begin{align}
\t_t^\star=\arg\max_{\t \in \Theta} F_\alpha\left(\frac{1}{m}\sum_{i=1}^{m} u_{ti}(\t)\right). \label{benchmark-theta}
\end{align}
\vspace{-0.1em}
Clearly, this benchmark is initially unknown and can only be calculated in hindsight (end of round). The RAF regret metric is in analogy to the \emph{task-average regret} in \cite{pmlr-v97-balcan19a}, \cite{khodak-neurips19}, which however do not consider multiple users or fairness.

\section{OWO Fair Learning Mechanism} \label{sec:algorithm}

Our aim is to design an algorithm that enables the system to perform as well as the ideal (but unknown) benchmarks, i.e., to achieve $\lim_{T\rightarrow \infty} \mathcal R_T=0$. To that end, we first transform the problem to make it amenable to learning, and then design a specialized two-layer (outer-inner) algorithm which is lightweight, hence suitable for a AI-RAN system.

\textbf{Problem Transformation}. The transformation is two-fold. First, we derive an upper bound of the {per-round fairness regret}, as a function of the initial model, and perform learning on that bound (\emph{outer-loop}). Secondly, we use a dual representation of the fairness function, so as to remove the coupling of decisions across slots, and design a primal-dual algorithm where the dual variables represent natively the training \emph{priority} of each user (\emph{inner-loop}). 

We start by defining a proxy function for $F_\alpha$:
\begin{align}
\Psi_{ti}(w,\theta) \doteq (-F_\alpha)^\star(w)-w^\top u_{ti}(\theta) \label{eq:conjugate-function}
\end{align}
where $w\!\in\![-1/u_{\min}^{\alpha}, -1/u_{\max}^{\alpha}]^K$ are the dual conjugate variables (one per user), $(-F_\alpha)^\star$ the convex conjugate of the negated fairness function, and the utilities are bounded as $u_{ti} \in [u_{\min}, u_{\max}]^K$. Using \cite[Lemma~1]{bib:tareq_fairness}, this function can be written analytically as:
\begin{align*}
\Psi_{ti}(w,\theta)\!=\!\!
\begin{cases}
  \!\sum_{k \in \mathcal{K}} \frac{\alpha(-w_k)^{1-1/\alpha}-1}{1-\alpha}\! -\! w^\top u_{ti}(\theta), \alpha\! \in\! \mathbb{R}_{\geq 0}\!\setminus\! \{1\} \\[0.9ex]
  \!\sum_{k \in \mathcal{K}} -\log(-w_k) \!-\! w^\top u_{ti}(\theta) - 1, \alpha = 1.
\end{cases}
\end{align*}
The practical value of this function stems from the conjugate equivalence property \cite{bib:tareq_fairness}, which allows us to write:
\begin{align}
\max_{\t \in \Theta}F_\alpha\big({u_{ti}}(\t)\big)=\max_{\t \in \Theta}\min_{w\in \c W}\Psi_{ti}(w,\t). \label{eq:biconjugate}
\end{align}
Thus, we can maximize the fairness function through a \emph{primal-dual} iteration, where we successively update the model $\theta$ and weights $w$ that track how fair, across users, are the updates (users' priorities). Also, unlike $F_\alpha$, \eqref{eq:conjugate-function} is linear on $\{u_{ti}\}$, thus we can average it over slots, a necessary condition for OCO. 

\textbf{Primal-Dual Algorithm}. After the transformation we can design the layered learning algorithm (see Fig.~\ref{fig:algorithm-mechanics}) where the outer-loop learns to initialize the model based on the results of past rounds, and the inner-loop updates the model with data and feedback from the users after each slot. This is an interleaved training and inference process. 

We begin our analysis with the inner-loop, where we update the model $\{\theta_{ti}\}_i$ and the users' \emph{priorities} $\{w_{ti}\}_i$, and our aim is to perform as well as the round-best model $\theta_t^\star$. Formally, we wish to minimize the \emph{primal} and \emph{dual regret}:
\begin{gather}
   R_t^\t\doteq\sum_{i=1}^m \Psi_{ti}(w_{ti},\theta_t^\star)-\sum_{i=1}^m \Psi_{ti}(w_{ti},\theta_{ti}), \nonumber \ \\[0.5ex]
   R_t^w\doteq\sum_{i=1}^m \Psi_{ti}(w_{ti},\theta_{ti})-\sum_{i=1}^m \Psi_{ti}(w_{t}^\star,\theta_{ti}), \quad \forall t\leq T.
   \nonumber
\end{gather}
For the primal space, we use an \emph{Online Gradient Ascent} (OGA) algorithm \cite{orabona-book} with initial point $\theta_{t0}=x_t$ and learning rate $\eta = D_{\Theta}^\star/(G_\Theta \sqrt{m})$, where $D_\Theta^\star = \max_{x,y \in \Theta^\star} ||x-y||_2$ denotes the diameter of the subspace $\Theta^\star$ in which the fairest models across all rounds lie\footnote{The parameter $D_\Theta^\star$ serves as a measure of task similarity. To simplify the theoretical analysis, we assume that this quantity is known and thus we set $\eta = D_\Theta^\star / (G\sqrt{m})$. If this is not the case, the guessing mechanism proposed in~\cite[Th.~2.1]{khodak-neurips19} can be adapted to our setting to address this issue, requiring only the knowledge of $D_\Theta$, where $D_\Theta^\star \le D_\Theta$. However, under this modification, the key theoretical insights remain unaffected, allowing a clearer presentation.}. The shared model is then updated at the start of each slot $(t,i)$ by:
\vspace{-1mm}
\begin{equation}\label{eq:primal-update}
\begin{aligned}
\theta_{ti} &= \Pi_{\Theta} \Big( \theta_{t, i-1} + \eta \text{ } g_{t,i-1}^{\theta} \Big)
\end{aligned}
\vspace{-0.5em}
\end{equation}
where $g_{ti}^\t\!\doteq\!\g_\t \Psi_{ti}(w_{ti},\t_{ti})\!=\! \!\sum_k w_{tik} \nabla \ell_{tik}(\theta_{ti})$. The regret of this update depends on the initial model \cite[Th.~2.13]{orabona-book}:
\vspace{-0.3em}
\begin{align}
R_t^\t\leq \frac{1}{2\eta}\|  x_t - \theta_t^\star\|_2^2 + \frac{\eta}{2} G_\Theta^2 \text{ } m \doteq U_t(x_t) \label{eq:outer-loop-regret}
\end{align}
\vspace{-0.2em}
in which we defined the constant $G_{\Theta}\geq ||g_{ti}^\theta||_2$. We also introduced $U_t(\cdot)$ to mark the dependence of regret on $x_t$. This function reveals how to expedite the in-round learning by selecting {initial} models, coupling the inner and outer loops. 

Specifically, since $U_t$ is a quadratic function of $x$, we can optimize it over rounds ({outer-loop}) using OGD:
\vspace{-0.3em}
\begin{align}
    x_{t+1} = \Pi_{\mathcal{X}} \left( x_t - \beta_t \text{ } g_t^x \right) \label{eq:outer-learning}
\end{align}
\vspace{-0.2em}
where $g_t^x \doteq \g_x U_t(x_t)\!=\!\left(x_t\!-\!\theta_t^{\star} \right)/\eta$ with $||g^x_t||_2 \leq D_\Theta/\eta$, and $\beta_t$ the learning step. The goal of this loop is to gradually pick initial models that perform as well as the benchmark $x^\star$, namely to minimize the outer regret:
\vspace{-0.3em}
\begin{align}
R_T^x \doteq \sumT U_t(x_t) - \sumT U_t(x^\star),
\nonumber
\end{align}
where $x^\star\!=\! \arg \min_{x \in \mathcal X} \sum_{t=1}^T \!U_t(x)$ is the best possible initialization. Since $U_t$ is $(1/\eta)$-strongly convex, we set $\beta_t\!=\!\eta/t$ to achieve a conveniently-logarithmic regret bound:
\vspace{-0.3em}
\begin{equation}
R_T^x \leq \frac{D_\Theta^2}{2 \eta} (1 + \log T) \label{eq:outer-upper-bound},
\vspace{-0.3em}
\end{equation}
where $D_\Theta$ is the decisions diameter, $D_\Theta\!=\!\max_{x,y \in \Theta} || x \!-\! y||_2$.

Finally, the dual updates balance the \emph{priorities} of users when training the model, by weighting the respective gradients, so as to be fair across all $m$ slots of each round. The learning of these weights is performed with strongly convex OGD:
\begin{align}
w_{ti}= \Pi_{\mathcal{W}} \left( w_{t,i-1} - \gamma_t \text{ } g_{t,i-1}^w \right) \label{eq:dual-update}
\end{align}
with $g_{ti}^w\!\doteq\!\g_w \Psi_{ti}(w_{ti},\t_{ti})\!=\!( (-w_{tik})^{-1/\alpha}\!-\!u_{tik}(\theta_{ti}))_{k \in \mathcal{K}}$ and $\gamma_t$ the dual learning rate. Taking into account that $||g_{ti}^w||_2 \leq G_{\mathcal{W}} \!=\! \sqrt{K} \max \Big\{1/u_{\min}^{{1/\alpha}}\! -\! u_{\min}, u_{\max} \!-\! 1/u_{\max}^{1/\alpha} \Big\}$, we get:
\begin{align}
R_t^w \leq \frac{G_{\mathcal{W}}^2 \text{ } \alpha}{2 u_{min}^{1+1/\alpha}} \left( 1 + \log m \right), \label{eq:dual-upper-bound}
\end{align}
 that depends logarithmically on the inner-loop horizon $m$.

The detailed steps of this process are summarized in \emph{Algorithm~\ref{algorithm}}. At the beginning of each inner-loop, given a predetermined initialization, the participating users first extract and send their locally computed feature representations to the shared model. Afterwards, the shared model (AI-RAN) processes each user’s features and sends the output of its last layer to the corresponding user-specific head so that each forward pass is completed. Local models are then updated using their individual losses, while the shared model aggregates the task-specific gradients in a fair manner to update its parameters within each slot. After all slots in a round are completed, the shared model performs an additional update, resulting in a more informative initialization for the next round.

\textbf{Complexity and Performance Analysis}. In terms of computations, the algorithm requires one outer-loop OGD update, and two first-order updates per slot within the round. Depending on the geometry of $\Theta$, these updates could be computed even in closed form. In terms of scalability, unlike some prior MTL works (e.g., \cite{pmlr-v162-navon22a}, \cite{fair-icml24}), our method avoids per-user gradient storage by directly computing the weighted gradient. For instance, if the model concerns a Deep Learning network, this is equivalent to performing a single backpropagation per iteration on the weighted aggregated loss, plus lightweight per-user scalar weighting.
%
%
The performance of our scheme is stated in the next theorem (proved in appendix):
\begin{theorem}\label{thm:regret-bound}
Alg. \ref{algorithm} achieves round-average fairness regret:
	\begin{align}
		\c R_T\!\leq\!\Bigg[ G_\Theta D_\Theta^\star \!+\! \frac{G_\Theta D_\Theta^2}{2 D_\Theta^\star} \frac{1 \!+\! \log T}{T} \Bigg] \frac{1}{\sqrt{m}} 
        \!+\! \frac{G_{\mathcal{W}}^2 \text{ } \alpha}{2 u_{min}^{1\!+\!1/\alpha}} \frac{1 \!+\! \log m}{m}\notag
	\end{align}
\end{theorem}
The leading $\mathcal{O}(1/\sqrt{m})$ term in this bound dictates the decay of regret with respect to $m$, while the secondary logarithmic term decreases faster and is therefore asymptotically negligible. The dependence on $T$ appears only through a logarithmically scaled $1/T$ factor, which remains uniformly bounded and thus refines the relevant constants without altering the asymptotic rate. Consequently, the learning performance is primarily driven by $m$, whereas larger horizons $T$ have marginal effect.

\vspace{-0.6em}
\begin{algorithm}[!ht]
	\caption{OWO-FMTL} \label{algorithm}
	\begin{algorithmic}[1]
		\Require{$\mathcal{X}$, $\Theta$, $\alpha \in \mathbb{R}_{\geq 0}$, $\mathcal{U} = \left[ u_{\min}, u_{\max} \right]^K$}
        \State $\c W = \left[-1/u_{\min}^{\alpha}, -1/u_{\max}^{\alpha}\right]^K$
        \Comment{\textit{{\scriptsize{Definition of the dual space}}}} 
        \State ${x}_1 \in \mathcal{X}$
		\Comment{\textit{{\scriptsize{Initial shared model}}}}
		\For{$t=1 \textbf{ to } T$}   \Comment{\textit{{\scriptsize{Outer-loop  learning}}}} 
            \State Initialize $\theta_{t0} = x_t, w_{t1}\in \c W$ \Comment{\textit{{\scriptsize{Initial weights per round}}}}
            \State Set steps $\eta = \frac{D_\Theta^\star}{G_\Theta \sqrt{m}}$, $\beta_t = \frac{\eta}{t}$, $\gamma_t = \frac{\alpha}{u_{\min}^{1+1/\alpha} \text{ } t}$
            \For{$i=1 \textbf{ to } m$ }
                \Comment{\textit{{\scriptsize{Inner-loop learning}}}}
                \State Receive $g_{ti}^\t$ and $g_{ti}^w$
                \State Update $\theta_{ti}$ using \eqref{eq:primal-update} and $w_{ti}$ using \eqref{eq:dual-update}	
            \EndFor
            \State Compute fairest model $\t_t^\star$ using \eqref{benchmark-theta}
            \State Compute next initial model $x_{t+1}$ using \eqref{eq:outer-learning}
    	\EndFor
	\end{algorithmic}
\end{algorithm}
\vspace{-0.6em}
Compared to a natural baseline of \emph{single-round learning} ({SRL}), where at the beginning of each round the shared model is trained from scratch, our regret bound scales as $\mathcal{O}(D_\Theta^\star/\sqrt{m})$, while {SRL} yields $\mathcal{O}(D_\Theta/\sqrt{m})$ as $T \to \infty$. Hence, Algorithm \ref{algorithm} achieves a multiplicative improvement by a factor $\alpha = D_\Theta/D_\Theta^\star$. Whenever $D_\Theta^\star \ll D_\Theta$, this improvement is substantial, whereas when $D_\Theta \approx D_\Theta^\star$, the two bounds are of the same order, and our method is only worse by a constant factor. Thus, our theory formalizes the underlying tradeoff: the proposed approach is advantageous when many rounds are available and cross-round similarity is present, while remaining robust when tasks are diverse.

\section{Performance Evaluation} \label{sec:results}

We evaluate \emph{OWO-FMTL} using both \emph{synthetic} and \emph{application-based} experiments. The synthetic setup examines its convergence, while the real deep-learning task compares its performance against benchmarks. We also consider two regimes: a \emph{stochastic} setting, where labels are drawn i.i.d. from a time-varying distribution, and an \emph{adversarial} setting, where labels alternate between extreme values within each round (``ping-pong'' sequence). This standard OCO adversarial pattern maximally challenges the learner~\cite{orabona-book}, serving as a worst-case test for our OWO-FMTL method.

\textbf{Kernel Sinusoidal Regression}. We first design an intuitive experiment to test OWO-FMTL in the \emph{convex} setting. We employ a linear regression model in a higher-dimensional feature space, as data are represented by a third-degree polynomial kernel $\mu(\tau) = [1,\, \tau,\, \tau^2,\, \tau^3] \in [0,1]^4$ for $\tau \in [0,1]$. The goal is to predict user labels by learning a common parameter vector $\theta \in \Theta = [-1,1]^4$. We consider the case of two users ($K = 2$), each learning a sinusoidal wave at each slot $i\leq m$, within a round $t \leq T$:
\vspace{-0.6em}
\begin{align*}
y_{ti1}(\tau_{ti1}) &= A_{t1} \sin(\omega_{t1} \tau_{ti1} + \phi_{t1}) + \varepsilon_{t1} \text{ and } \\[4pt]
y_{ti2}(\tau_{ti2}) &= A_{t2} \cos(\omega_{t2} \tau_{ti2} + \phi_{t2}) + \varepsilon_{t2},
\end{align*}
\vspace{-0.2em}
with $A_{t1} \sim \mathcal{N}(1, 0.2)$, $A_{t2} \sim \mathcal{N}(0.6, 0.5)$, 
$\omega_{t1} \sim \mathcal{N}(1, 1)$, $\omega_{t2} \sim \mathcal{N}(0.75, 1)$, 
$\phi_{t1} \sim \mathcal{N}(\pi/3, 0.01)$, $\phi_{t2} \sim \mathcal{N}(\pi/4, 0.01)$, 
and $\varepsilon_{t1}, \varepsilon_{t2} \overset{\text{i.i.d.}}{\sim} \mathcal{N}(0.05, 0.1)$.

We study an adversarial setting, where for all of $T$ rounds, the sign of the true label of the first user is alternated every $\sqrt{m}$-sized blocks following a round-dependent pattern. We also evaluate the algorithm in a stochastic setting, where no such sign alternations occur. For both cases, we consider $T\!=\!512$ rounds, a varying number of slots per round, $m \in \{4, 8, 16, 32, 64, 128\}$ and two different fairness notions corresponding to $\alpha \in \{1,2\}$.
\vspace{-1.4em}
\begin{figure}[!ht]
    \centering
    \includegraphics[width=0.8\columnwidth]{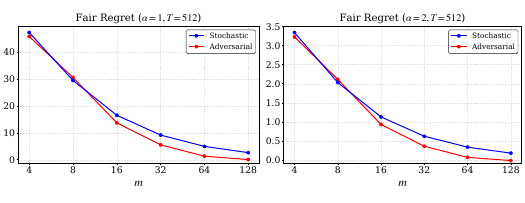}
    \vspace{-5mm}
    \caption{Sublinear fairness regret of OWO-FMTL with respect to $m$, for $\alpha = 1$ (left) and $\alpha = 2$ (right), under both stochastic and adversarial environments.}
    \label{fig:fair_regret_combined}
\end{figure}
\vspace{-0.9em}

As shown in Fig.~\ref{fig:fair_regret_combined}, the fairness regret decreases \emph{sublinearly} as the number of slots per round $m$ increases. For small values of $m$, the stochastic and adversarial curves nearly coincide, whereas for larger values of $m$, OWO-FMTL achieves lower fairness regret.

\textbf{Deep-Learning Digit Recognition}.
Next, we evaluate OWO-FMTL in a \emph{non-convex} setting with two users, by training a LeNet Convolutional Neural Network (CNN) \cite{lecun2002gradient} on the MNIST dataset. We construct distinct tasks for the users across rounds by applying the transformations proposed in~\cite{pmlr-v97-finn19a}, which yield the Rainbow MNIST dataset, comprising 56 tasks (7 background colors, 2 scales, 4 orientations) and a total of 68{,}600 samples, with 1{,}225 images per task. We repeat this process two times so as to increase the number of training rounds, each time re-splitting the dataset into 56 new non-overlapping tasks; this yields $T\!=\!168$ rounds. Furthermore, to simulate the MTL environment at each slot within a round, users are assigned distinct tasks that change from round to round but remain fixed within each round. Finally, we introduce an adversarial perturbation mechanism whereby every $\sqrt{m}$ slots within each round, the true labels of the second user are corrupted by replacing each label $y$ with $9 - y$, while labels in all other slots retain their true values.

The experiments model consists of a \emph{shared feature extractor} and two \emph{task-specific} heads, one per user. As mentioned earlier, the shared encoder follows a LeNet-style architecture, consisting of two convolutional layers with ReLU activations and MaxPooling after each layer, followed by a fully connected layer with 50 units. Each task-specific head includes a single fully connected layer producing 10-class outputs (one for each digit) and employs \emph{cross-entropy} loss. The model is trained using SGD implemented in PyTorch with fairness parameter $\alpha = 1$ (proportional fairness), utility bounds $u_{\min} = 1$ and $u_{\max} = 3.3$, and inner/outer-loop learning rates of 0.1. The batch size for each user is set to 25, resulting in $m\!=\!31$ slots per task.

\vspace{-3mm}
\begin{figure}[!ht]
    \centering
\includegraphics[width=0.7\columnwidth]{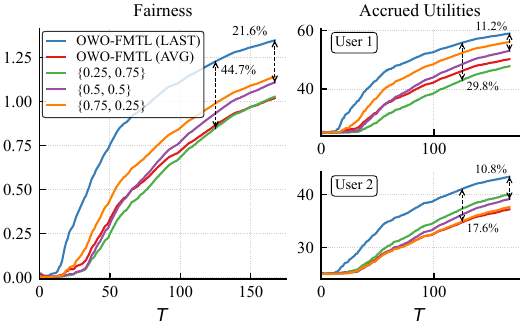}
    \vspace{-4mm}
    \caption{Accumulated fairness (left) and utilities (right) across rounds.} 
    \vspace{-0.6em}
    \label{fig:fairness_utilities_icc}
\end{figure}
Since computing $\theta_t^\star$ in this non\mbox{-}convex setting is impractical, we approximate it using either the parameters from the last slot ({LAST}) of the round or the average parameters ({AVG}) across the round. In terms of fairness, Fig.~\ref{fig:fairness_utilities_icc} shows that OWO-FMTL (LAST) achieves the best trade-off between \emph{fairness} and user \emph{utilities} across training rounds, compared to different \emph{constant weighting schemes} (CWS), including the standard MTL approach that simply averages the task losses. As CWS approaches rely on fixed task weights, their performance depends on whether the chosen weights align with the optimal ones. This limitation becomes evident in practice, as CWS exhibit larger fairness and utility gaps, particularly for the adversity-affected second user. In contrast, our dynamic task-weighting approach resolves this issue, achieving $\approx20$--$40\%$ higher fairness and $\approx10$--$30\%$ higher utilities across users.
\vspace{-1em}
\begin{figure}[!ht]
    \centering
    \includegraphics[width=0.6\columnwidth]{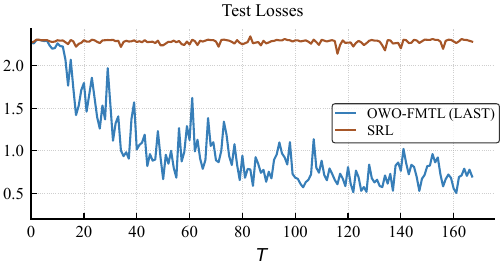}
    \vspace{-4mm}
    \caption{Test losses on the validation data of each round.}
    \label{fig:test_losses_icc}
\end{figure}

\vspace{-0.6em}
Regarding the impact of the outer-loop, Fig.~\ref{fig:test_losses_icc} illustrates the evolution of \emph{test losses} on unseen validation data across rounds. When the model is trained from scratch at each round (SRL), the short horizon limits learning, keeping test loss near random guessing. In contrast, our method steadily reduces test loss, indicating that the outer-loop learns meaningful initializations that enable rapid adaptation to each round’s tasks. A similar result is expected without adversarial perturbations. However, the fairness–utility tension is more pronounced under highly volatile and adversarial conditions, making the learning process even more challenging.
\section{Conclusions and Future Work} \label{sec:conclusions}
Future mobile networks aim to deliver AI services through edge computing at the RAN, where MTL offers an efficient way to address limited resources. To this end, we proposed OWO-FMTL, a lifelong extension of MTL that balances fairness and learning efficiency, and achieves an $\mathcal{O}(1/\sqrt{m})$ vanishing round-average fair regret in the convex setting. Numerical results further demonstrate consistent gains over representative baselines across diverse operating conditions. As future work, we plan to extend our evaluation to real-world wireless datasets to better reflect practical communication environments and constraints.

\section*{Acknowledgments} \label{sec:acknowledgments}
This work has been supported by the National Growth Fund through the Dutch 6G flagship project “Future Network Services”, and the European Commission through Grant No. 101139270 (ORIGAMI), 101168816 (FINALITY) and 101192462 (FLECON-6G).
\section{Appendix} \label{sec:appendix}
\textbf{Proof of Th. 1}. Rearranging the primal regret bound in \eqref{eq:outer-loop-regret}:
\begin{align*}
	&\frac{1}{m}\summ \Psi_{ti}(  w_{ti},   \t_{ti}) + \frac{U_t(x_t)}{m} \geq \frac{1}{m}\summ \Psi_t( w_{ti}, \t_t^\star)\\ &=\frac{1}{m}\summ \Big((-F_a)^\star( w_{ti}) -  w_{ti}^\top {u}_{ti}( \t_t^\star)	\Big) \notag\\
	&\stackrel{(i)}\geq (-F_a)^\star({\bar w_t})\! -\! \frac{{\bar w_t}^\top}{m}\summ {u}_{ti}( \t_t^\star)\! -\! \frac{1}{m}\summ ( w_{ti}\! -\! \Bar{ w_t})^\top\! u_{ti}({\t_t}^\star)      \notag \\
	&\stackrel{(ii)}\geq F_a\left( \frac{1}{m}\summ u_{ti}( \t_t^\star)\right)  - \frac{1}{m}\summ ( w_{ti} - \Bar{ w_t})^\top u_{ti}({\t_t}^\star) 
\end{align*}
where $(i)$ uses Jensen's inequality and $(ii)$ the bi-conjugate equivalence. Next, using the dual regret bound \eqref{eq:dual-upper-bound} the definition of $ \Psi_{ti}$ in \eqref{eq:conjugate-function}, we can write:
\begin{align*}
	&\frac{1}{m}\summ \Psi_{ti}(  w_{ti},   \t_{ti}) \leq \frac{R_t^w}{m} +  \frac{1}{m}\summ \Psi_{ti}(  w_{t}^\star,   \t_{ti})\\
	&=\frac{R_t^w}{m} + F_a\left( \frac{1}{m} \summ u_{ti}(\theta_{ti}) \right).
\end{align*}
Putting these together, we get:
\begin{equation*}
\begin{aligned}
&F_a\!\left( \frac{1}{m}\sum_{i=1}^m u_{ti}(\theta_t^\star)\right)
 - F_a\!\left( \frac{1}{m}\sum_{i=1}^m u_{ti}(\theta_{ti}) \right) \leq \\
&\quad \hspace{-0.8cm} \frac{U_t(x_t)}{m} 
   + \frac{R_t^w}{m} + \frac{1}{m}\sum_{i=1}^m ( w_{ti} - \bar{w}_t )^\top u_{ti}(\theta_t^\star)
\end{aligned}
\label{eq:ineq}
\end{equation*}
where the last residual term, under the standard perturbation-restrictiveness assumption on the environment of \cite[Th.~1]{bib:tareq_fairness}, tends to $\mathit{o}_m(1)$. Otherwise achieving vanishing horizon-fairness regret is not possible. Thus, the overall regret is:
\vspace{-0.5em}
\begin{equation*}
\begin{aligned}
\mathcal{R}_T
&\leq \frac{1}{T}\sum_{t=1}^T\left[ \frac{1}{m} \left(\frac{1}{2 \eta}\| \theta_t^\star - x_t\|_2^2 + \frac{\eta G_\Theta^2 \text{ } m}{2} \right)\right] + \\[0.4ex]
&\hspace{-5mm} \frac{G_{\mathcal{W}}^2 \text{ } \alpha}{2 u_{min}^{1+1/\alpha}} \frac{1 + \log m}{m} \leq \frac{1}{mT}\sum_{t=1}^T \left[ \frac{1}{2 \eta}\| \theta_t^\star - x^\star\|_2^2 + \frac{\eta G_\Theta^2}{2} \right] \\[0.4ex]
& +\frac{G_{\mathcal{W}}^2 \text{ } \alpha}{2 u_{min}^{1+1/\alpha}} \frac{1 + \log m}{m} + \frac{D_\Theta^2}{2 m \eta} \frac{1 + \log T}{T}
\end{aligned}
\vspace{-0.4em}
\end{equation*}
where the last inequality follows from \eqref{eq:outer-upper-bound}. Setting $\eta\!=\! D_\Theta^\star/(G_\Theta \sqrt{m})$ leads to desired bound.

\bibliographystyle{IEEEtran}
\bibliography{ref-MTL-icc.bib}
\end{document}